\documentclass{article}
\usepackage{arxiv}
\usepackage{amsfonts,amsmath,amssymb,amsthm}
\usepackage[ruled, vlined, noend]{algorithm2e}
\usepackage{array}
\usepackage{bm,bbm}
\usepackage{booktabs}
\usepackage{graphicx}
\usepackage{here}
\usepackage{float}
\usepackage{framed}
\usepackage{latexsym}
\usepackage{multirow}
\usepackage{natbib}
\usepackage{overpic}
\usepackage{hyperref}
\usepackage{url}
\usepackage{xcolor}
\bibpunct[, ]{(}{)}{;}{; }{,}{,}
\hypersetup{bookmarksnumbered=true, bookmarksopen=true, colorlinks=true, linkcolor=blue, citecolor=blue,}
\newtheorem{theorem}{Theorem}
\SetKwRepeat{Do}{do}{while}
\newcommand{\zo}{{\rm zo}}
\newcommand{\ab}{{\rm ab}}
\newcommand{\sq}{{\rm sq}}
\newcommand{\tot}{{\rm tot}}

\newcommand{\para}{{\rm par}}
\newcommand{\quo}{{\rm quo}}
\newcommand{\rem}{{\rm rem}}
\newcommand{\thr}{{\rm thr}}
\DeclareMathOperator{\argmax}{{arg\,max}}
\DeclareMathOperator{\argmin}{{arg\,min}}
\newcommand{\bb}{{\bm{b}}}
\newcommand{\bc}{{\bm{c}}}
\newcommand{\bt}{{\bm{t}}}
\newcommand{\bx}{{\bm{x}}}
\newcommand{\bX}{{\bm{X}}}
\newcommand{\calY}{{\mathcal{Y}}}
\DeclareMathOperator{\bbE}{{\mathbb{E}}}
\DeclareMathOperator{\bbI}{{\mathbbm{1}}}
\newcommand{\bbR}{{\mathbb{R}}}
\newcommand{\bbN}{{\mathbb{N}}}
\newcolumntype{C}[1]{>{\hfil}m{#1}<{\hfil}}

\newcommand{\scs}{\scriptsize}

\title{Parallel Algorithm for Optimal Threshold Labeling of Ordinal Regression Methods}

\author{Ryoya Yamasaki
\And
Toshiyuki Tanaka}
\begin{document}
\maketitle
\begin{abstract}
Ordinal regression (OR) is classification of ordinal data 
in which the underlying categorical target variable has 
a natural ordinal relation for the underlying explanatory variable.
For $K$-class OR tasks, threshold methods learn 
a one-dimensional transformation (1DT) of the explanatory variable 
so that 1DT values for observations of the explanatory variable preserve the order of 
label values $1,\ldots,K$ for corresponding observations of the target variable well,
and then assign a label prediction to the learned 1DT through threshold labeling,
namely, according to the rank of an interval to which the 1DT belongs 
among intervals on the real line separated by $(K-1)$ threshold parameters.
In this study, we propose a parallelizable algorithm 
to find the optimal threshold labeling, 
which was developed in previous research,
and derive sufficient conditions for that algorithm to 
successfully output the optimal threshold labeling.
In a numerical experiment we performed, the computation time 
taken for the whole learning process of a threshold method 
with the optimal threshold labeling could be 
reduced to approximately 60\,\% 
by using the proposed algorithm with parallel processing compared 
to using an existing algorithm based on dynamic programming.
\end{abstract}
\section{Introduction}
\label{sec:Introduction}
\emph{Ordinal regression} (\emph{OR}, or called ordinal classification) is classification of 
\emph{ordinal data} in which the underlying target variable is categorical and considered to 
be equipped with a \emph{natural ordinal relation} for the underlying explanatory variable.
Various applications, including
age estimation \citep{niu2016ordinal, cao2020rank, yamasaki2022optimal}, 
information retrieval \citep{liu2009learning},
movie rating \citep{yu2006collaborative}, and
questionnaire survey \citep{burkner2019ordinal},
leverage OR techniques.

\emph{Threshold methods} (or called threshold models)
have been actively studied in machine learning research
\citep{shashua2003ranking, lin2006large, chu2007support, lin2012reduction, 
li2007ordinal, pedregosa2017consistency, yamasaki2022optimal},
and they are popularly used for OR tasks as 
a simple way to capture the ordinal relation of ordinal data.
Those methods learn a \emph{one-dimensional transformation} (\emph{1DT}) of 
the observation of the explanatory variable so that an observation 
with a larger class label tends to have a larger 1DT value;
they predict a label value through \emph{threshold labeling},
that is, according to the rank of an interval 
to which the learned 1DT belongs among intervals on 
the real line separated by $(K-1)$ \emph{threshold parameters}
for a $K$-class OR task.

\citet{yamasaki2022optimal} proposed to use 
a threshold parameter vector that minimizes the empirical task risk,
and showed experimentally that the corresponding threshold labeling function 
(\emph{optimal threshold labeling}) could lead to better classification performance 
than other labeling functions applied in previous OR methods 
\citep{mccullagh1980regression, shashua2002taxonomy, chu2005new, 
lin2006large, lin2012reduction, pedregosa2017consistency, cao2020rank}.
The previous study \citep{yamasaki2022optimal} employs 
a \emph{dynamic-programming-based (DP) algorithm} \citep{lin2006large} 
for optimization of the threshold parameter vector, but that algorithm does not suit 
acceleration based on parallel processing and requires long computation time 
when the training data size and number of classes are large.
In order to mitigate this trouble, in this study, we propose 
another parallelizable algorithm, which we call 
\emph{independent optimization} (\emph{IO}) \emph{algorithm},
to find the optimal threshold labeling, 
and derive sufficient conditions for that algorithm 
to successfully output the optimal threshold labeling.
Moreover, we demonstrate, through a numerical experiment, 
that the paralleled IO algorithm can reduce the computation time.

The rest of this paper is organized as follows.
Section~\ref{sec:Preparation} introduces formulations of the
OR task, threshold method, and optimal threshold labeling,
in preparation for discussing calculation algorithms 
for the optimal threshold labeling.
Section~\ref{sec:DP} provides a review of an
existing calculation algorithm, the DP algorithm.
Section~\ref{sec:IO} presents the main results of this paper,
IO algorithm and its parallelization procedure and theoretical properties.
Section~\ref{sec:Experiment} gives an experimental 
comparison of the computational efficiency of the DP, 
non-parallelized IO, and parallelized IO algorithms.
Section~\ref{sec:Conclusion} concludes this paper.
Additionally, in Appendix~\ref{sec:FP}, we mention 
another parallelization procedure of the IO algorithm, 
which is more complex than that described in the main text.

\section{Preparation}
\label{sec:Preparation}
\subsection{Formulation of Ordinal Regression Task}
Suppose that we are interested in behaviors of the categorical target variable 
$Y\in[K]:=\{1,\ldots,K\}$ given a value of the explanatory variable $\bX\in\bbR^d$,
and that we have ordinal data $(\bx_1,y_1),\ldots,(\bx_n,y_n)\in\bbR^d\times[K]$
that are considered to be 
independent and identically distributed observations of the pair $(\bX,Y)$
and have a natural ordinal relation like examples described 
in Section~\ref{sec:Introduction}, with $K, d, n\in\bbN$ such that $K\ge3$.%
\footnote{%
In this paper, we do not discuss what the natural ordinal relation is. 
Refer to, for example, \citet{da2008unimodal, yamasaki2022unimodal} for this topic.}
An OR task is a classification task of the ordinal data.
In this paper, we focus on a popular formulation of the OR task,
searching for a classifier $f:\bbR^d\to[K]$ that is good 
in minimization of the \emph{task risk} $\bbE[\ell(f(\bX),Y)]$ for 
a user-specified \emph{task loss function} $\ell:[K]^2\to[0,+\infty)$, 
where the expectation $\bbE[\cdot]$ is taken for $(\bX,Y)$.
Popular task loss functions include not only 
the zero-one task loss $\ell_\zo(k,l):=\bbI(k\neq l)$ 
(for minimization of misclassification rate), but also V-shaped losses 
(for cost-sensitive classification tasks) reflecting the user's preference 
of smaller prediction errors over larger ones such as 
the absolute task loss $\ell_\ab(k,l):=|k-l|$,
and squared task loss $\ell_\sq(k,l):=(k-l)^2$, 
where $\bbI(c)$ is 1 if the condition $c$ is true and 0 otherwise.

In a formulation of the OR task, one may adopt 
other criteria that cannot be decomposed into 
a sum of losses for each data point: 
for example, quadratic weighted kappa 
\citep{cohen1960coefficient,cohen1968weighted}. 
Our discussion in this paper does not cover such criteria.

\subsection{Formulation of Threshold Method}
\label{sec:FTM}
\emph{1DT-based methods} \citep{yamasaki2022optimal} learn 
a 1DT $a:\bbR^d\to\bbR$ of the explanatory variable $\bX$ so that learned
1DT values $\hat{a}(\bx_i)$'s preserve the order of label values $y_i$'s well
\citep{mccullagh1980regression, shashua2002taxonomy, chu2005new, 
lin2006large, lin2012reduction, pedregosa2017consistency, 
cao2020rank, yamasaki2022unimodal, yamasaki2022optimal}.
For example, \emph{ordinal logistic regression (OLR)} 
\citep{mccullagh1980regression} models 
the probability of $Y=y$ conditioned on $\bX=\bx$ by
$P(y,a(\bx),\bb)=\sigma(b_1-a(\bx))$ (for $y=1$),
$\sigma(b_y-a(\bx))-\sigma(b_{y-1}-a(\bx))$ (for $y\in\{2,\ldots,K-1\}$),
$1-\sigma(b_{K-1}-a(\bx))$ (for $y=K$) with a 1DT $a$, 
a bias parameter vector $\bb=(b_k)_{k\in[K-1]}\in\bbR^{K-1}$ 
satisfying $b_1\le\cdots\le b_{K-1}$,
and the sigmoid function $\sigma(u):=1/(1+e^{-u})$,
and it can learn $a$ and $\bb$ via, 
for example, the maximum likelihood method
$(\hat{a},\hat{\bb})\in\argmax_{a,\bb}\prod_{i=1}^nP(y_i,a(\bx_i),\bb)$.
1DT-based methods construct a classifier $f$ as $f=h\circ\hat{a}$ 
with a learned 1DT $\hat{a}$ and a \emph{labeling function} $h:\bbR\to[K]$.
Many of existing 1DT-based methods can be seen 
as adopting a \emph{threshold labeling function},
\begin{equation}
\label{eq:ThrLab}
  h_\thr(u;\bt):= 1+\sum_{k=1}^{K-1}\bbI(u\ge t_k)
\end{equation}
with a threshold parameter vector $\bt=(t_k)_{k\in[K-1]}\in\bbR^{K-1}$;
we call such classification methods threshold methods.

\subsection{Optimal Threshold Labeling}
\label{sec:OPT}
\citet{yamasaki2022optimal} proposed to use a threshold parameter vector 
$\hat{\bt}$ that minimizes the empirical task risk 
for a learned 1DT $\hat{a}$ and specified task loss $\ell$:
\begin{equation}
\label{eq:OPT}
  \hat{\bt}\in\underset{\bt\in\bbR^{K-1}}{\argmin}\,
  \frac{1}{n}\sum_{i=1}^n\ell(h_\thr(\hat{a}(\bx_i);\bt), y_i).
\end{equation}
%
It was also shown experimentally that, for various learning methods of the 1DT, 
the optimal threshold labeling function $h_\thr(u;\hat{\bt})$ could yield smaller 
test task risk than other labeling functions used in previous studies
\citep{mccullagh1980regression, shashua2002taxonomy, chu2005new, 
lin2006large, lin2012reduction, pedregosa2017consistency, cao2020rank},
for example, $h(u)=h_\thr(u;\hat{\bb})$ with a learned bias parameter vector 
$\hat{\bb}$ and likelihood-based labeling function 
$h(u)\in\argmin_{k\in[K]}(\sum_{l=1}^K\ell(k,l)P(l,u,\hat{\bb}))$
for OLR reviewed in Section~\ref{sec:FTM}.

It is important to reduce the computation time required to optimize the threshold parameter vector.
For example, when a user learns the 1DT with an iterative optimization algorithm 
and employs early stopping \citep{prechelt2002early}, the user evaluates the empirical task risk 
using a holdout validation dataset at every epoch.
For this purpose it is necessary to calculate the optimal threshold parameter vector
for a 1DT available at each epoch, and the computation time for this calculation can account for 
a high percentage of that for the whole learning process for a threshold method,
which will be shown by a numerical experiment we took (see Table~\ref{tab:CT}).
Therefore, we study computationally efficient algorithms for 
calculating the optimal threshold parameter vector.

\section{Existing Dynamic-Programming-based (DP) Algorithm}
\label{sec:DP}
It was shown in \citet[Theorem 6]{yamasaki2022optimal} 
that the minimization problem \eqref{eq:OPT} can be solved by 
the DP algorithm (Algorithms~\ref{alg:Preparation} and \ref{alg:DP-algorithm}),
which was developed by \citet{lin2006large}, as
\begin{equation}
\begin{split}
  &\hat{\bt}\in\underset{\bt\in\bbR^{K-1}}{\argmin}\,\frac{1}{n}\sum_{i=1}^n\ell(h_\thr(\hat{a}(\bx_i);\bt), y_i),\\
  &\text{s.t.~}t_1,\ldots,t_{K-1}\in\{c_j\}_{j\in[N+1]}\text{~and~}t_1\le\cdots\le t_{K-1}
\end{split}
\end{equation}
with the \emph{candidate vector} $\bc=(c_j)_{j\in[N+1]}$.
Here, $N$ is the total number of the sorted unique elements 
$\hat{a}_j$, $j=1,\ldots,N$ of $\{\hat{a}(\bx_i)\}_{i\in[n]}$ 
that satisfy $\hat{a}_1<\cdots<\hat{a}_N$ and 
$\hat{a}(\bx_i)\in\{\hat{a}_j\}_{j\in[N]}$ for all $i\in[n]$,
and we define $\bc$ by $c_j:=-\infty$ (for $j=1$), 
$(\hat{a}_{j-1}+\hat{a}_j)/2$ (for $j\in\{2,\ldots,N\}$), $+\infty$ (for $j=N+1$).
Note that there is a degree of freedom in definition of the candidate vector $\bc$,
and that the value of the training task risk does not change 
even if $c_1$, $c_j$ for $j\in\{2,\ldots,N\}$, and $c_{N+1}$ 
are replaced with a smaller value than $\min(\hat{a}_j)_{j\in[N]}$,
a value in $(\hat{a}_{j-1},\hat{a}_j)$, 
and a larger value than $\max(\hat{a}_j)_{j\in[N]}$.

Algorithm~\ref{alg:DP-algorithm} does not suit parallel processing,
and takes the computation time of the order $O(N\cdot K)$.
Note that Algorithm~\ref{alg:Preparation} takes the computation time 
of the order $O(n\log n)$ in average by, for example, quick sort, 
but the actual computation time for Algorithm~\ref{alg:DP-algorithm} 
can be longer than that for Algorithm~\ref{alg:Preparation};
see experimental results in Section~\ref{sec:Experiment}.

{\IncMargin{.5em}\begin{algorithm}[t]
\caption{Preparation for Algorithms~\protect\ref{alg:DP-algorithm}, \protect\ref{alg:IO-algorithm}, and \protect\ref{alg:PO-algorithm}}
\label{alg:Preparation}\SetAlgoNlRelativeSize{0}
\KwIn{Training data $\{(\bx_i,y_i)\}_{i\in[n]}$, learned 1DT $\hat{a}$, and task loss $\ell$.}
\lnl{P1} Calculate $\hat{a}(\bx_i)$, $i=1,\ldots,n$\;
\lnl{P2} Calculate sorted unique elements $\hat{a}_j$, $j=1,\ldots,N$ of $\{\hat{a}(\bx_i)\}_{i\in[n]}$\;
\lnl{P3} Calculate sets $\calY_j=\{y_i\mid \hat{a}(\bx_i)=\hat{a}_j\}_{i\in[n]}$, $j=1,\ldots,N$\;
\tcc{Calculate loss matrix $(M_{j,k})_{j\in[N],k\in[K]}$. Parallelizable with $j\in[N]$ and $k\in[K]$.}
\lnl{P4} \lFor{$j=1,\ldots,N$ and $k=1,\ldots,K$}{$M_{j,k}=\sum_{y_i\in\calY_j}\ell(k,y_i)$}
\lnl{P5} Define the candidate vector $\bc$ as $c_j:=-\infty$ (for $j=1$), 
$(\hat{a}_{j-1}+\hat{a}_j)/2$ (for $j\in\{2,\ldots,N\}$), $+\infty$ (for $j=N+1$)\;
\KwOut{$\hat{a}_j$ and $\calY_j$ ($j=1,\ldots,N$), loss matrix $M$, and candidate vector $\bc$.}
\end{algorithm}\DecMargin{.5em}}
%

{\IncMargin{.5em}\begin{algorithm}[t]
\caption{DP algorithm to calculate the optimal threshold parameter vector\protect\footnotemark}
\label{alg:DP-algorithm}\SetAlgoNlRelativeSize{0}
\KwIn{$\hat{a}_j$ and $\calY_j$ ($j=1,\ldots,N$), loss matrix $M$, and candidate vector $\bc$ prepared by Algorithm~\ref{alg:Preparation}.}
\tcc{Calculate $(L_{j,k})_{j\in[N],k\in[K]}$ sequentially.}
\lnl{DP1} \lFor{$k=1,\ldots,K$}{$L_{1,k}=M_{1,k}$}
\lnl{DP2}\For{$j=2,\ldots,N$}{
  \tcc{An efficient implementation of $L_{j,k}=\min (L_{j-1,l})_{l\in[k]}+M_{j,k}$.}
  \lnl{DP3} $O_j\leftarrow+\infty$\;
  \lnl{DP4} \For{$k=1,\ldots,K$}{
    \lnl{DP5} \lIf{$L_{j-1,k}<O_j$}{$O_j\leftarrow L_{j-1,k}$}
    \lnl{DP6} $L_{j,k}=O_j+M_{j,k}$\;
  }
}
\tcc{Calculate threshold parameters $(\hat{t}_k)_{k\in[K-1]}$ sequentially.}
\lnl{DP7} $I\leftarrow \min(\argmin(L_{N,k})_{k\in[K]})$\;
\lnl{DP8} \lIf{$I\neq K$}{$\hat{t}_k=c_{N+1}$ for $k=I,\ldots,K-1$}
\lnl{DP9} \For{$j=N-1,\ldots,1$}{
  \lnl{DP10} $J\leftarrow \min(\argmin(L_{j,k})_{k\in[I]})$\;
  \lnl{DP11} \lIf{$I\neq J$}{$\hat{t}_k=c_{j+1}$ for $k=J,\ldots,I-1$, and $I\leftarrow J$}
}
\lnl{DP12} \lIf{$I\neq 1$}{$\hat{t}_k=c_1$ for $k=1,\ldots,I-1$}
\KwOut{A threshold parameter vector $(\hat{t}_k)_{k\in[K-1]}$.}
\end{algorithm}\DecMargin{.5em}}
\footnotetext{\label{fn:DP}%
Optimality is guaranteed regardless of 
which element of $\argmin(L_{N,k})_{k\in[K]}$ is set to $I$ in Line~\ref{DP7} and
which element of $\argmin(L_{j,k})_{k\in[I]}$ is set to $J$ in Line~\ref{DP10}.
This is also true for Line~\ref{IO4} of Algorithm~\ref{alg:IO-algorithm} 
and Line~\ref{PO10} of Algorithm~\ref{alg:PO-algorithm} 
described in Appendix~\ref{sec:FP}.}

\section{Proposed Independent Optimization (IO) Algorithm}
\label{sec:IO}
We here propose to learn the threshold parameters $t_k$, $k\in[K-1]$ independently by 
Algorithm~\ref{alg:IO-algorithm} (IO algorithm), 
which is parallelizable with $k\in[K-1]$.
This algorithm is developed according to the conditional relation
\begin{equation}
\label{eq:IOTEQ}
  \frac{1}{n}\sum_{i=1}^n \ell(h_\thr(\hat{a}(\bx_i); \bt), y_i)\\
   =\sum_{k=1}^{K-1}R_k(t_k)
  -\underbrace{\sum_{k=2}^{K-1}R_k(+\infty)}_{\text{$\bt$-independent}}
  \text{ if } t_1\le\cdots\le t_{K-1}
\end{equation}
with the functions $R_k$, $k\in[K-1]$ defined by
\begin{equation}
  R_k(t)
  :=\frac{1}{n}\sum_{i=1}^n \ell(h_\thr(\hat{a}(\bx_i); (\cdots,-\infty, \underbrace{t}_{\text{$k$-th}}, +\infty,\cdots)), y_i),
\end{equation}
which is the empirical task risk when it labels $\hat{a}(\bx_i)<t$ as $k$ and $\hat{a}(\bx_i)\ge t$ as $(k+1)$.
Figure~\ref{fig:IO} schematizes Equation~\eqref{eq:IOTEQ}:
the summation of the red parts implies the left-hand side term of \eqref{eq:IOTEQ}, 
and the summation of the blue parts implies the latter term of the right-hand side of \eqref{eq:IOTEQ}.
This relation implies the equivalence between minimizing the empirical 
task risk $\frac{1}{n}\sum_{i=1}^n \ell(h_\thr(\hat{a}(\bx_i); \bt), y_i)$
and minimizing $R_k(t_k)$, $k=1,\ldots,K-1$ independently for each $k$
when the threshold parameters satisfy the order condition $t_1\le\cdots\le t_{K-1}$.
Therefore, Algorithm~\ref{alg:IO-algorithm} solves 
the latter subproblems associated with each threshold parameter independently.

{\IncMargin{.5em}\begin{algorithm}[t]
\caption{IO algorithm to calculate the optimal threshold parameter vector}
\label{alg:IO-algorithm}\SetAlgoNlRelativeSize{0}
\KwIn{$\hat{a}_j$ and $\calY_j$ ($j=1,\ldots,N$), loss matrix $M$, and candidate vector $\bc$ prepared by Algorithm~\ref{alg:Preparation}.}
\tcc{Parallelizable with $k\in[K-1]$.}
\lnl{IO1} \For{$k=1,\ldots,K-1$}{
  \tcc{Calculate $(R_{j,k})_{j\in[N+1]}$. Further parallelizable (see Appendix~\ref{sec:FP}).}
  \lnl{IO2} $R_{1,k}=0$\;
  \lnl{IO3} \lFor{$j=1,\ldots,N$}{$R_{j+1,k}=R_{j,k}+M_{j,k}-M_{j,k+1}$}
  \tcc{Calculate threshold parameter $\hat{t}_k$.}
  \lnl{IO4} $\hat{t}_k=c_{j_k}$ with $j_k=\min(\argmin(R_{j,k})_{j\in[N+1]})$\;
}
\KwOut{A threshold parameter vector $(\hat{t}_k)_{k\in[K-1]}$.}
\end{algorithm}\DecMargin{.5em}}
%

\begin{figure}[t]
\centering%
\includegraphics[width=12cm, bb=0 0 454 165]{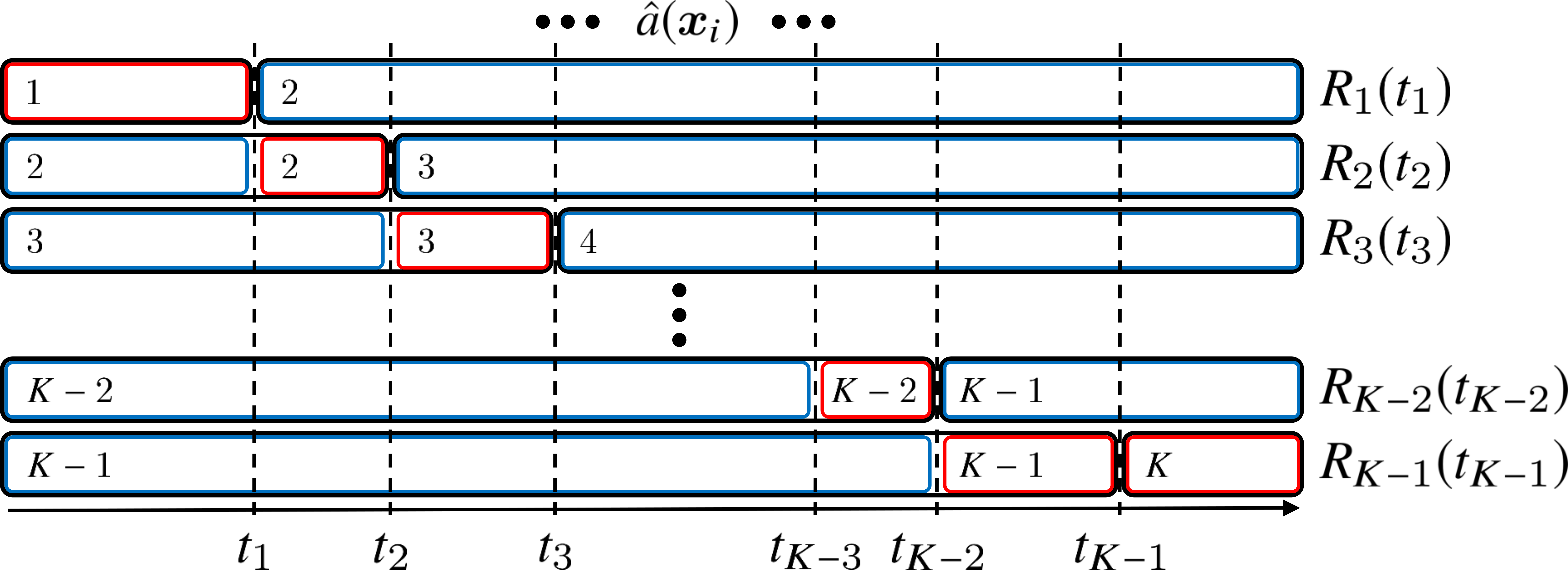}
\caption{%
This figure assists in understanding Equation~\eqref{eq:IOTEQ}.
In calculation of $R_k(t_k)$,
the learned 1DT value $\hat{a}(\bx_i)$ located in the range marked with $l$ is labeled as $l$, 
and the corresponding task loss is $\ell(l,y_i)$.}
\label{fig:IO}
\end{figure}

As the condition of Equation~\eqref{eq:IOTEQ} also suggests,
the threshold parameter vector $\hat{\bt}$ obtained by 
Algorithms~\ref{alg:Preparation} and \ref{alg:IO-algorithm} 
becomes optimal in minimization of the empirical task risk 
as long as the learned vector $\hat{\bt}$ eventually follows 
the appropriate order $\hat{t}_1\le\cdots\le\hat{t}_{K-1}$:
\begin{theorem}
\label{thm:Optimal}
For any data $\{(\bx_i,y_i)\}_{i\in[n]}$, learned 1DT $\hat{a}$, and task loss $\ell$, 
the threshold parameter vector $\bt=\hat{\bt}$ obtained by 
Algorithms~\ref{alg:Preparation} and \ref{alg:IO-algorithm} 
minimizes the empirical task risk $\frac{1}{n}\sum_{i=1}^n \ell(h_\thr(\hat{a}(\bx_i); \bt), y_i)$,
if it satisfies the order condition $\hat{t}_1\le\cdots\le \hat{t}_{K-1}$.
\end{theorem}
\begin{proof}[Proof of Theorem~\ref{thm:Optimal}]
In this proof, we use the notations $\hat{a}_j$, $R_{j,k}$, and $j_k$, defined 
in Algorithms~\ref{alg:Preparation} and \ref{alg:IO-algorithm} as well.
Because of the definition, it holds that
\begin{equation}
\label{eq:IOTEQ2}
  R_{j,k}=n\{R_k(c_j)-R_k(c_1)\}
  \text{~for all~}j\in[N+1],\,k\in[K-1].
\end{equation}
According to Equations \eqref{eq:IOTEQ} and \eqref{eq:IOTEQ2},
one has that
\begin{equation}
\label{eq:sumLikeq}
  \begin{split}
  &\frac{1}{n}\sum_{i=1}^n \ell(h_\thr(\hat{a}(\bx_i); (c_{i_k})_{k\in[K-1]}, y_i)\\
  &=\sum_{k=1}^{K-1}R_k(c_{i_k})
  -\sum_{k=2}^{K-1}R_k(+\infty)
  \quad(\because\eqref{eq:IOTEQ})\\
  &=\sum_{k=1}^{K-1}\biggl\{\frac{1}{n}R_{i_k,k}+R_k(c_1)\biggr\}
  -\sum_{k=2}^{K-1}R_k(+\infty)
  \quad(\because\eqref{eq:IOTEQ2})\\
  &=\frac{1}{n}\sum_{k=1}^{K-1}R_{i_k,k}
  +\text{($(i_k)_{k\in[K-1]}$-independent term)}
  \end{split}
\end{equation}
with indices $i_1,\ldots,i_{K-1}\in[N+1]$ satisfying $i_1\le \cdots\le i_{K-1}$.
Minimization of $\sum_{k=1}^{K-1}R_{i_k,k}$ regarding $i_1,\ldots,i_{K-1}$ amounts to 
minimization of the empirical task risk with the threshold parameters $(c_{i_k})_{k\in[K-1]}$ regarding $i_1,\ldots,i_{K-1}$
as far as the solution of the former problem keeps the ascending order.
Algorithm~\ref{alg:IO-algorithm} finds the index $j_k$ that minimizes $(R_{j,k})_{j\in[N+1]}$ for each $k\in[K-1]$.
Namely, $(j_k)_{k\in[K-1]}$ is a minimizer of $\sum_{k=1}^{K-1}R_{i_k,k}$ regarding $i_1,\ldots,i_{K-1}$.
Therefore, under the assumption of this theorem 
($\hat{t}_1\le\cdots\le\hat{t}_{K-1}$ or equivalently $j_1\le \cdots\le j_{K-1}$),
the threshold parameter vector $\hat{\bt}=(c_{j_k})_{k\in[K-1]}$ minimizes the empirical task risk. 
\end{proof}

Furthermore, we found that the assumption of Theorem~\ref{thm:Optimal} 
holds for the task with a convex task loss, 
which is defined by being convex regarding the first augment,
for example, $\ell=\ell_\ab, \ell_\sq$:
\begin{theorem}
\label{thm:Convex}
For any data $\{(\bx_i,y_i)\}_{i\in[n]}$ and learned 1DT $\hat{a}$, 
the threshold parameter vector $\bt=\hat{\bt}$ obtained by Algorithms~\ref{alg:Preparation} and \ref{alg:IO-algorithm} 
satisfies the order condition $\hat{t}_1\le\cdots\le \hat{t}_{K-1}$ and hence minimizes 
the empirical task risk $\frac{1}{n}\sum_{i=1}^n \ell(h_\thr(\hat{a}(\bx_i); \bt), y_i)$,
if the task loss function $\ell$ satisfies 
\begin{equation}
\label{eq:Convex}
  \ell(k,l)-2\ell(k+1,l)+\ell(k+2,l)\ge0
  \text{~for all~}k\in[K-2],\,l\in[K].
\end{equation}
\end{theorem}
\begin{proof}[Proof of Theorem~\ref{thm:Convex}]
In this proof, we use the notations $\hat{a}_j$, $\calY_j$, $M_{j,k}$, $R_{j,k}$, and $j_k$, 
defined in Algorithms~\ref{alg:Preparation} and \ref{alg:IO-algorithm} as well.
We prove this theorem using the proof by contradiction.
Assume $\hat{t}_k>\hat{t}_{k+1}$ (i.e., $j_k>j_{k+1}$) with some $k\in[K-2]$.
One has that
\begin{equation}
\label{eq:Comp}
  \begin{split}
  &\{R_{j_k,k}+R_{j_{k+1},k+1}\}-\{R_{j_{k+1},k}+R_{j_k,k+1}\}\\
  &=\biggl\{\sum_{j<j_k}(M_{j,k}-M_{j,k+1})+\sum_{j<j_{k+1}}(M_{j,k+1}-M_{j,k+2})\biggr\}\\
  &-\biggl\{\sum_{j<j_{k+1}}(M_{j,k}-M_{j,k+1})+\sum_{j<j_k}(M_{j,k+1}-M_{j,k+2})\biggr\}
  \quad(\because\text{Algorithm~\ref{alg:IO-algorithm}, Line~\ref{IO3}})\\
  &=\sum_{j<j_k}(M_{j,k}-2M_{j,k+1}+M_{j,k+2})
  -\sum_{j<j_{k+1}}(M_{j,k}-2M_{j,k+1}+M_{j,k+2})\\
  &=\sum_{j\in\{j_{k+1},\ldots,j_k-1\}}(M_{j,k}-2M_{j,k+1}+M_{j,k+2})\\
  &=\sum_{j\in\{j_{k+1},\ldots,j_k-1\}}\sum_{y_i\in\calY_j}\{\ell(k,y_i)-2\ell(k+1,y_i)+\ell(k+2,y_i)\}
  \quad(\because\text{Algorithm~\ref{alg:Preparation}, Line~\ref{P4}})\\
  &\ge0\quad(\because\eqref{eq:Convex}).
  \end{split}
\end{equation}
When the equality holds in the last inequality of \eqref{eq:Comp},
$R_{j_k,k}+R_{j_{k+1},k+1}=R_{j_{k+1},k}+R_{j_k,k+1}$ also holds.
Furthermore, $R_{j_k,k}\le R_{j_{k+1},k}$ and $R_{j_{k+1},k+1}\le R_{j_k,k+1}$ hold because 
$j_k$ minimizes $(R_{j,k})_{j\in[N+1]}$ and $j_{k+1}$ minimizes $(R_{j,k+1})_{j\in[N+1]}$.
From these results, it necessarily holds that $R_{j_k,k}=R_{j_{k+1},k}$ and $R_{j_{k+1},k+1}=R_{j_k,k+1}$,
which contradicts the definition $j_k=\min(\argmin(R_{j,k})_{j\in[N+1]})=\min(\{j_{k+1}, j_k,\ldots\})$ since $j_k>j_{k+1}$.
When the equality does not hold in the last inequality of \eqref{eq:Comp},
$R_{j_k,k}+R_{j_{k+1},k+1}>R_{j_{k+1},k}+R_{j_k,k+1}$ holds, and hence 
the threshold parameter vector $(\check{t}_k)_{k\in[K-1]}$ consisting of 
$\check{t}_k=\hat{t}_{k+1}$, $\check{t}_{k+1}=\hat{t}_k$,
and $\check{t}_j=\hat{t}_j$ for $j\neq k, k+1$ is better than $\hat{\bt}$ 
in minimization of $\sum_{k=1}^{K-1}R_{i_k,k}$ regarding $i_1,\ldots,i_{K-1}$.
According to the sub-optimality of $\hat{\bt}$,
it can be seen that Algorithm~\ref{alg:IO-algorithm} does not output $\hat{\bt}$,
which is a contradiction to the setting of this proof.
Therefore, the proof is concluded.
\end{proof}

On the ground of Theorems~\ref{thm:Optimal} and \ref{thm:Convex},
it is guaranteed that one can obtain the optimal threshold parameter vector efficiently 
by the \emph{parallelized IO} (\emph{PIO}) \emph{algorithm} 
(Algorithm~\ref{alg:IO-algorithm}, 
or Algorithm~\ref{alg:PO-algorithm} in Appendix~\ref{sec:FP})
for the task with a convex task loss function.

\section{Experiment}
\label{sec:Experiment}
In order to verify that the proposed PIO algorithm can be faster than the existing DP algorithm, 
we performed a practical OR experiment using real-world datasets to compare 
the computational efficiency of the DP, (non-parallelized) IO, and PIO algorithms.

We addressed an OR task of estimating age from facial image by 
a discrete-value prediction, with reference to experiment settings in 
the previous OR studies \citep{cao2020rank, yamasaki2022optimal}.
The used datasets are MORPH-2 \citep{ricanek2006morph}, 
CACD \citep{chen2014cross}, and AFAD \citep{niu2016ordinal}.
After pre-processing as that in \citet{cao2020rank, yamasaki2022optimal}, 
we used data of the total data size $n_\tot=55013, 159402, 164418$ 
and $K=55, 49, 26$ in these datasets.
For each dataset, we resized all images to $128\times 128$ pixels of 3 RGB channels 
and divided the dataset into a 90\,\% training set and a 10\,\% validation set.
As a learning method of the 1DT, 
we tried OLR that we reviewed in Section~\ref{sec:Preparation}.
We implemented the 1DT with ResNet18 or ResNet34 \citep{he2016deep}, 
and trained the network using the negative-log-likelihood loss function 
and Adam with the mini-batch size 256 for 100 epochs.
For the task with the absolute task loss $\ell=\ell_\ab$, 
we optimized the threshold parameters by the DP, IO, or PIO algorithms 
and calculated the empirical task risk for the validation set, 
validation error (VE), at every epoch.
We used images randomly cropped to $120 \times 120$ pixels of 3 channels 
in `training of 1DT', and images center-cropped to the same size in 
`preparation (Algorithm~\ref{alg:Preparation})', 
`optimization of threshold parameters', and `calculation of VE', as inputs.

We experimented with Python 3.8.13 and PyTorch 1.10.1
under the computational environment with 
4 CPUs Intel Xeon Silver 4210 and a single GPU GeForce RTX A6000.
We performed `training of 1DT', 
`Line~\ref{P1} of preparation', 
and `calculation of VE' using the GPU, 
and `remaining lines of preparation' 
and `optimization of threshold parameters' using the CPUs,
which was the best allocation of the CPUs and GPU among those we tried.
We employed the parallelization with {\tt DataLoader} 
of {\tt num\_workers} 8 for the calculations on the GPU, 
$K$-parallelization for `Line~\ref{P4} of preparation' 
and $(K-1)$-parallelization for `Lines~\ref{IO1}--\ref{IO4} 
of Algorithm~\ref{alg:IO-algorithm}' of the PIO
among the calculations on the CPUs.
See \url{https://github.com/yamasakiryoya/ECOTL} for program codes we used.
%

\begin{table}[t]
\renewcommand{\tabcolsep}{5pt}
\centering%
\caption{Calculation time for 100 epochs (in minutes).}\label{tab:CT}
\scalebox{0.8}{
\begin{tabular}{cc|C{1.6cm}|C{1.6cm}|C{1.6cm}}
\toprule
&&MORPH-2&CACD&AFAD\\
\midrule
\multirow{2}{*}{Training of 1DT}
&{\scs `ResNet18,\,GPU' or}&32.00&86.59&83.46\\
&{\scs `ResNet34,\,GPU'}&45.27&125.27&121.54\\
\midrule
&{\scs `Line\,\ref{P1},\,ResNet18,\,GPU'\,or}&23.22&48.11&58.30\\
Preparation
&{\scs `Line\,\ref{P1},\,ResNet34,\,GPU',}&21.33&50.61&59.46\\
(Algorithm~\ref{alg:Preparation})
&{\scs `Lines\,\ref{P2} and \ref{P3},\,CPUs',\,and}&1.05&3.61&3.63\\
&{\scs `Lines\,\ref{P4} and \ref{P5},\,CPUs'}&9.66&13.60&9.94\\
\midrule
\multirow{2}{*}{Optimization of}
&{\scs `DP,\,CPUs',}&81.04&186.60&130.35\\
\multirow{2}{*}{threshold parameters}
&{\scs `IO,\,CPUs',\,or}&87.71&224.09&123.55\\
&{\scs `PIO,\,CPUs'}&18.45&24.38&15.26\\
\midrule
\multirow{2}{*}{Calculation of VE}
&{\scs `ResNet18,\,GPU'\,or}&5.20&8.40&9.68\\
&{\scs `ResNet34,\,GPU'}&5.07&9.18&9.85\\
\bottomrule
\end{tabular}}
\end{table}

Table~\ref{tab:CT} shows the calculation time taken for each process.
The total calculation time with the DP, IO, and PIO algorithms
($t_{\rm DP}$, $t_{\rm IO}$, and $t_{\rm IO}$) had the following relationship:
the ratios $(t_{\rm IO}/t_{\rm DP}, t_{\rm PIO}/t_{\rm DP})$ are
$(1.04,0.59)$ for MORPH-2, $(1.11,0.53)$ for CACD, and $(0.98,0.61)$ for AFAD for the ResNet18-based 1DT;
$(1.04,0.62)$ for MORPH-2, $(1.10,0.58)$ for CACD, and $(0.98,0.66)$ for AFAD for the ResNet34-based 1DT.
Numerical evaluations related to the computation time may change 
slightly depending on a used computational environment,
but we consider that the following facts will hold universally:
the DP and IO algorithms have comparable computational efficiency,
while the PIO algorithm can be faster than the former two algorithms.

\section{Conclusion}
\label{sec:Conclusion}
In this paper, we proposed the parallelizable IO algorithm (Algorithm~\ref{alg:IO-algorithm})
for acquiring the optimal threshold labeling \citep{yamasaki2022optimal}, 
and derived sufficient conditions (Theorems~\ref{thm:Optimal} and \ref{thm:Convex})
for this algorithm to successfully output the optimal threshold labeling.
Our numerical experiment showed that the computation time 
for the whole learning process of a threshold method with 
the optimal threshold labeling was reduced to approximately 
60\,\% under our computational environment, 
by using the parallelized IO algorithm compared to using the existing DP algorithm.
On the ground of these results, 
we conclude that the proposed algorithm can be useful 
to accelerate the application of threshold methods 
especially for cost-sensitive tasks with a convex task loss, 
which are often tackled in OR research.



\bibliographystyle{abbrvnat}
\bibliography{multitask_learning, ordinal_regression, SVM_SVR, machine_learning, sparse, age_estimation}

\appendix
\section{Further Parallelization of IO Algorithm}
\label{sec:FP}
The IO algorithm (Algorithm~\ref{alg:IO-algorithm}) can be parallelized further.
Algorithm~\ref{alg:PO-algorithm} describes a parallelization procedure 
regarding $k\in[K-1]$ and $l\in[L^\para]$ with 
$L^\para=\lceil (N+1)/L\rceil$ for a user-specified integer $L$.
This algorithm calculates the translated version of $L^\para$ blocks 
splitting $(R_{j,k})_{j\in[N+1]}$ (Lines~\ref{PO2}--\ref{PO5}) and 
then modifies the translation of each block (Lines~\ref{PO6}--\ref{PO9}) 
in parallel, to calculate the whole of $(R_{j,k})_{j\in[N+1]}$,
instead of Lines~\ref{IO2} and \ref{IO3} of Algorithm~\ref{alg:IO-algorithm},
for each $k\in[K-1]$.
Algorithms~\ref{alg:IO-algorithm} and \ref{alg:PO-algorithm} yield the same outputs.
If one has sufficient computing resources, 
it would be better to set $L$ to an integer of the order $O(\sqrt{N})$
theoretically since this parallelized IO algorithm takes 
computation time of the order $O(L+L^\para)$,
and this implementation improves the order of computation time from 
$O(N\cdot K)$ of the DP algorithm (Algorithms~\ref{alg:DP-algorithm}) to $O(\sqrt{N})$.
However, we note that it is not realistic to expect as dramatic an improvement in actual computation time as the order evaluation suggests, 
owing to the limitation of computing resources:
under the setting of the experiment in Section~\ref{sec:Experiment},
Algorithm~\ref{alg:PO-algorithm} with 
$L=\lceil (N+1)/2\rceil, \lceil (N+1)/3\rceil, \lceil (N+1)/4\rceil$ 
(such that $L^\para=2,3,4$) and CPUs takes 
$37.27,52.07,67.08$ for MORPH-2, $40.38,51.98,64.95$ for CACD, and 
$23.33,28.92,35.43$ for AFAD (in minutes) for `optimization of threshold parameters' for 100 epochs.
`PIO, CPUs' following Algorithm~\ref{alg:IO-algorithm} was better
than Algorithm~\ref{alg:PO-algorithm} under our computational environment.

{\IncMargin{.5em}\begin{algorithm}[t]
\caption{Further parallelized IO algorithm to calculate the optimal threshold parameter vector{\protect\footnotemark}}
\label{alg:PO-algorithm}\SetAlgoNlRelativeSize{0}
\KwIn{$\hat{a}_j$ and $\calY_j$ ($j=1,\ldots,N$), loss matrix $M$, candidate vector $\bc$ prepared by Algorithm~\ref{alg:Preparation},
$L\in\bbN$, $L^\para:=\lceil (N+1)/L\rceil$, $L^\quo:=\lfloor (N+1)/L\rfloor$, $L^\rem:=(N+1)-L\cdot L^\quo$, and $L_l:= (l-1)L$.}
\tcc{Parallelizable with $k\in[K-1]$.}
\lnl{PO1} \For{$k=1,\ldots,K-1$}{
  \tcc{Calculate $(Q_{j,k})_{j\in[N+1]}$. Parallelizable with $l\in[L^\para]$.}
  \lnl{PO2} \For{$l=1,\ldots,L^\para$}{
    \lnl{PO3} $Q_{L_l+1,k}=0$\;
    \lnl{PO4} \lIf*{$l\le L^\quo$}{\lFor{$j=1,\ldots,L-1$}{$Q_{L_l+j+1,k}=Q_{L_l+j,k}+M_{L_l+j,k}-M_{L_l+j,k+1}$}}
    \lnl{PO5} \lElse*{\lFor{$j=1,\ldots,L^\rem-1$}{$Q_{L_l+j+1,k}=Q_{L_l+j,k}+M_{L_l+j,k}-M_{L_l+j,k+1}$}}
  }
  \tcc{Calculate $(R_{j,k})_{j\in[N+1]}$. Parallelizable with $l\in[L^\para]$.}
  \lnl{PO6} \For{$l=1,\ldots,L^\para$}{
    \lnl{PO7} $S_{k,l}=\bbI(l\neq1)\sum_{m\in[l-1]}(Q_{L_{m+1},k}-Q_{L_m+1,k}+M_{L_{m+1},k}-M_{L_{m+1},k+1})$\;
    \lnl{PO8} \lIf*{$l\le L^\quo$}{\lFor{$j=1,\ldots,L$}{$R_{L_l+j,k}=Q_{L_l+j,k}+S_{k,l}$}}
    \lnl{PO9} \lElse*{\lFor{$j=1,\ldots,L^\rem$}{$R_{L_l+j,k}=Q_{L_l+j,k}+S_{k,l}$}}
  }
  \tcc{Calculate threshold parameter $\hat{t}_k$.}
  \lnl{PO10} $\hat{t}_k=c_{j_k}$ with $j_k=\min(\argmin(R_{j,k})_{j\in[N+1]})$\;
}
\KwOut{A threshold parameter vector $(\hat{t}_k)_{k\in[K-1]}$.}
\end{algorithm}\DecMargin{.5em}}
\footnotetext{
$\lceil\cdot\rceil$ is the ceiling function, and $\lfloor\cdot\rfloor$ is the floor function.}


\end{document}